\documentclass{article}
\PassOptionsToPackage{numbers, sort&compress}{natbib}
\usepackage[final]{amlc_guide}
\usepackage{algorithmic}
\usepackage{algorithm}

\usepackage{amssymb, amsthm, amsmath, graphicx,bm, paralist,subfigure,pdfpages, appendix}
\usepackage{hyperref}
\newtheorem{assumption}{Assumption}
\newtheorem{lem}{Lemma} \newtheorem{thm}{Theorem}
\newtheorem{defn}{Definition}

\def\mc{\mathcal}
\def\mb{\mathbf}
\def\mbb{\mathbb}
\def\ra{\rightarrow}
\def\bmu{\boldsymbol{\mu}}
\def\bnu{\boldsymbol{\nu}}

\makeatletter
\renewcommand{\@fnsymbol}[1]{}  
\makeatother

\title{Multi-robot Path Planning and Scheduling via\\ Model Predictive Optimal Transport (MPC-OT)}

\author{%
Usman A. Khan$^{1,2,\dagger,\ddagger}$\thanks{$^\dagger$UAK holds concurrent appointments as a Professor of CS at Boston College and as an Amazon Scholar. This paper describes work performed at Amazon and is not associated with Boston College.}
\thanks{$^\ddagger$This is the author’s version of the work accepted for publication in \emph{Proceedings of the 2025 IEEE Conference on Decision and Control}. The final accepted version will be available at IEEE Xplore.}
\quad Mouhacine Benosman$^1$ \quad  Wenliang Liu$^1$\\ 
\textbf{Federico Pecora}$^1$\quad \textbf{Joseph W. Durham}$^{1}$\\
$^{1}$Amazon Robotics, North Reading, MA\\
$^2$Computer Science, Boston College, Chestnut Hill, MA\\
\tt{\{uakhan,mbenos,liuwll,fpecora,josepdur\}@amazon.com}
}%

\begin{document}
\maketitle
\thispagestyle{empty}
\pagestyle{empty}

\begin{abstract}
In this paper, we propose a novel methodology for path planning and scheduling for multi-robot navigation that is based on optimal transport theory and model predictive control. We consider a setup where~$N$ robots are tasked to navigate to~$M$ targets in a common space with obstacles. Mapping robots to targets first and then planning paths can result in overlapping paths that lead to deadlocks. We derive a strategy based on optimal transport that not only provides minimum cost paths from robots to targets but also guarantees non-overlapping trajectories. We achieve this by discretizing the space of interest into~$K$ cells and by imposing a~${K\times K}$ cost structure that describes the cost of transitioning from one cell to another. Optimal transport then provides \textit{optimal and non-overlapping} cell transitions for the robots to reach the targets that can be readily deployed without any scheduling considerations. The proposed solution requires~$\mathcal O(K^3\log K)$ computations in the worst-case and~$\mathcal O(K^2\log K)$ for well-behaved problems. To further accommodate potentially overlapping trajectories (unavoidable in certain situations) as well as robot dynamics, we show that a temporal structure can be integrated into optimal transport with the help of \textit{replans} and \textit{model predictive control}.

\end{abstract}

\section{INTRODUCTION}
Robot-to-target assignment and navigation is an important problem in multi-robot swarms, automated warehouses, and multi-armed robotic manipulations~\cite{DeepFleet}. In these problems,~$N$ robots are tasked to navigate to~$M$ targets in a common space with obstacles. This problem is also known as MAPF (multi-agent path finding)~\cite{Stern2019}, where robot-target pairs are given, and the task is to navigate the entire fleet of robots to their pre-assigned targets. A variant, anonymous MAPF~\cite{Ali2023}, is where any robot-target pairing is valid and we seek the pairing with minimal cost (travel, makespan) for the entire fleet. Anonymous MAPF (AMAPF) is tractable under certain assumptions, akin to solving a linear program, whereas general MAPF is significantly more challenging, requiring an integer linear program due to the terminal constraints.

In this paper, we study the AMAPF problem, which is highly relevant in scenarios where it is sufficient for any robot to reach any target. A plausible approach for AMAPF is to solve an assignment problem over an~${N\times M}$ cost matrix, whose elements are the minimum costs to travel between all robot-target pairs. This approach, although simple, has several disadvantages: 
\begin{inparaenum}[(i)]
    \item needs pre-computed minimum travel costs between all robot-target pairs;
    \item each min-cost path is computed independently and the procedure is susceptible to collisions, bottlenecks, and deadlocks;
    \item the search space is restricted to exactly one min-cost path for every pair and thus ignores other paths of potentially the same cost. 
\end{inparaenum}

In contrast, we propose an optimal transport (OT) framework, formulated over a~$K$-cell discretization of the underlying space~$\Omega$, that does not rely on pre-computed paths. In particular, we construct a~${K\times K}$ cost matrix~$C$ of single cell transitions, that captures every possible path (sequence of transitions) between all robot-target pairs. The resulting transport~$\Pi$ is optimized over~$C$ and therefore over all possible paths between all robot-target pairs. The  transport~$\Pi$ must adhere to the marginal distribution constraints, which lead to non-overlapping paths. In situations where overlapping paths are unavoidable, we propose replans where the OT problem is solved again on a finer discretization of~$\Omega$, equivalent to increasing travel capacity of the cells. Combining replans with MPC (model-predictive control), we show that each robot is able to track the OT trajectories, within a bounded error, and eventually converges exponentially to the final OT trajectory leading them to the desired targets.

While OT has a long history, it has gained renewed attention due to its applications in machine learning and data science~\cite{Villani2003,Peyre2019,Figalli2022}. Its flexibility and mathematical rigor have also motivated exploration in multi-agent path planning (PP) problems. Existing work includes~\cite{Bandyopadhyay2014,deBadyn2018} that consider density control of networked multi-agent systems. Refs.~\cite{Krishnan2018, Frederick2022} prescribe transport PDEs and gradient flows for mult-robot swarms, while~\cite{kachar2022dynamic} integrates robot dynamics and LQR control costs in the~${N\times N}$ OT cost matrix. Related work also includes:~\cite{Emerick2023}, on modeling swarms as continuum densities;~\cite{Le2023}, on entropic-regularized formulations; and more recently~\cite{Liang2024}, which utilizes diffusion models in continuous spaces. See also~\cite{Stern2019} for a comprehensive survey on MAPF, and~\cite{Ma2016, Andreychuk2022, Andreychuk2023, Fine2023} for some recent results on anonymous MAPF. In contrast, the proposed OT-based approach is formulated such that the entire space~$\Omega$ is reconfigured from a starting distribution to the desired distribution, while naturally resulting in non-overlapping trajectories due to the marginal distribution constraints on the transport plan. 

We further enhance our approach by incorporating real-time replanning, refining the space to account for cell travel capacity, and integrating temporal trajectory tracking using model predictive control (MPC)~\cite{Rossiter2004,Benosman2016}. The receding horizon nature of MPC interleaved with the OT planner constitutes an attempt to integrate time constraints in OT, which is a notoriously hard problem~\cite{arxiv_K_SHI}. Moreover, the proposed OT formulation opens up new directions for MAPF by e.g., adding robot dynamics and collision avoidance via appropriate control actions, while also offering the flexibility to shape paths with varying geometry and density through the cost matrix or additional constraints. Furthermore, it bridges MAPF with recent advances in OT, such as low-rank approximations, entropic smoothing, and gradient flow formulations, which can be employed to develop scalable solutions and other modern variations.

We now describe the rest of the paper. Section~\ref{sec_pf} provides the problem formulation and recaps OT. Section~\ref{sec_ot} describes our main contribution on using OT for joint matching and path planning. We consider replans and MPC-based OT in Sections~\ref{sec_replans} and~\ref{sec_mpc}, respectively. Finally, Section~\ref{sec_sims} provides numerical simulations and Section~\ref{sec_conc} concludes the paper. 


\section{PROBLEM FORMULATION}\label{sec_pf}
Consider~$N$ robots and~$M$ targets located in a region of interest~${\Omega\subset\mathbb R^2}$ with~${\kappa\subset\Omega}$ being the set of obstacles. The starting (centroid) location of robot~$n$ is denoted by~$\mathbf p_n$ and the~$m$-th target location is given by~$\mathbf r_m$, with~$n\in[1,\ldots,N],m\in[1,\ldots,M]$. 
We assume that each robot~$n$ has the following dynamics~(${t\geq0}$):
\begin{align}\label{ss1}
    \dot{\mathbf{x}}_n(t) &= f_n(\mb x_n(t)) + g_n(\mb x_n(t))\cdot \mb u_n(t),\\\label{ss2}
    \mb y_n(t) &= h_n(\mb x_n(t)),
\end{align}
where~$\mb x_n(t)$ is the state of robot~$n$,~$\mb u_n(t)$ is the control input, and~$\mb y_n(t)$ is its location, with~${\mb y_n(0) = \mb p_n}$. Let~$q_{nm}(t)$ be a path that starts at robot~$n$ and ends at target~$m$. Let~${L=\min(N,M)}$, our goal is to design~$L$ control inputs that map distinct robot-target pairs in a cost-optimal way:
\begin{align}
    \min_{n,m }\sum_{l=1}^L\int_0^T c(q_{nm}(t)) dt,
\end{align}
where~$c(\cdot)$ denotes the cost to traverse a path and may include both the path length and the control cost. Since we are dealing with AMAPF, the minimization is over all robot-target pairs. Ideally, we would like to design trajectories that are non-overlapping and therefore each robot can charter its own course without any scheduling considerations. 


This is a challenging problem overall as it requires minimizing the total cost (path length plus control inputs) over all distinct robot-target pairs. In this paper, we break this problem into three phases that are discussed in the following three sections: 
\begin{inparaenum}[(i)] 
\item \textbf{Section~\ref{sec_ot}--OT-based matching and path planning:} Find \textit{non-overlapping} minimum cost paths that match~$L$ distinct robots to $L$ distinct targets such that no other robot-target pair could result in a lower cost. We achieve this with the help of optimal transport without including the control cost in the first phase; 
\item \textbf{Section~\ref{sec_replans}--replans:} When overlapping paths are acceptable and can be implemented, e.g., through scheduling, we propose finer discretizations, effectively increasing travel capacity in certain lanes. To implement this, we decompose the paths obtained from OT into shorter segments and solve a series of OT problems for each segment;
\item \textbf{Section~\ref{sec_mpc}--Integrating MPC:} Finally, we combine controller design based on the robot kinematics \eqref{ss1}-\eqref{ss2} with the paths obtained from OT replans in order to add both physical and temporal constraints. 
\end{inparaenum}
Before we proceed, we briefly recap optimal transport.

\subsection{Optimal Transport}
In 1781, the French mathematician Gaspard Monge formulated the problem of transporting mass from a source destination to a target destination that can be cast in terms of probability measures. Consider a source measure~${\mu:\mc A\ra\mbb R_+}$ and a target measure~${\nu:\mc B\ra\mbb R_+}$, where~$\mu(A)$ tells us how much mass we have in the set~${A\subseteq\mc A}$, with~${\mu(\mc A) = \nu(\mc B)=1}$. Optimal transport seeks a map~${T:\mc A \ra \mc B}$ such that
\begin{align*}
\mu(T^{-1}(B)) := \mu(\{a\in\mc A|T(a)\in B\}) = \nu(B),~\forall B\subseteq \mc B,
\end{align*}
where~$\mu(T^{-1}(B))$ is the push-forward measure~$T\sharp\mu$ of~$\mu$ through~$T$~\cite{Villani2003,Peyre2019,Figalli2022}. Monge's OT formulation is given by
\begin{align}\label{monge_eq}
\inf_T \int_{\mc A}c(a,T(a)) d\mu(a)\qquad\mbox{ s.t. }T\sharp\mu=\nu.
\end{align}
An issue with Monge's formulation is that a feasible map~$T$ may not exist because it does not allow mass to be split. For example, consider the source measure~$\mu$ to be a delta function and the target measure~$\nu$ to be some continuous measure, then~$T$ can only push-forward the source delta to another delta and thus a feasible solution does not exist. 

In 1939, Leonid Kantorovich provided an alternate formulation that does not consider a transport map but instead seeks a transport \textit{plan} that allows mass to be split and go to different places. A transport plan~$\Pi(A,B)$ describes how much mass goes from~${A\subseteq\mc A}$ to~${B\subseteq \mc B}$, which in fact is a measure on the product space~${\mc{A\times B}}$. We would like mass conservation that is given by~$
{\Pi(A,\mc B) = \mu(A)}, {\Pi(\mc A, B) = \nu(B)}$,
$\forall {A\subseteq\mc A},{B\subseteq\mc B}$. These constraints focus only on those measures, denoted by~$\Pi(\mu,\nu)$, that have the correct marginals instead of all possible joint measures~$\Pi(\mc {A,B})$. We now consider a cost~$c(a,b)$ of going from~${a\in\mc A}$ to~${b\in\mc B}$ weighted further by the amount of mass moving from~$a$ to~$b$. Formally, Kantorovich's transport problem is given by
\begin{align}\label{kant_eq}
\inf_{\pi(a,b)\in\Pi(\mu,\nu)} \int_{\mc{A\times B}} c(a,b) d\pi(a,b).
\end{align}

When both the source and target measures are discrete resources (as in robots being matched to targets), OT takes a very convenient form. Consider two discrete densities:
\[
\mu(a) = \sum_{i=1}^n \mu_i\delta_{\mathbf p_i}(a),\qquad \nu(b) = \sum_{j=1}^m \nu_j\delta_{\mathbf r_j}(b),
\]
where~${\delta_{\mb p_i}(a)=+\infty}$, when~${a=\mb p_i}$, and~$0$ otherwise. A transport plan is now a matrix~${\Pi=\{\pi_{ij}\}}$ with elements~${\pi_{ij}\geq0}$ that tells us how much mass from~$\mb p_i$ is moved to the location~$\mb r_j$, such that mass conservation is not violated. Discrete OT is formally stated as
\begin{align*}
&\inf_{\Pi=\{\pi_{ij}\}} \sum_{i}\sum_j \pi_{ij}\cdot c(\mb p_i,\mb r_j), \\
\mbox{s.t.~~}&\sum_{j=1}^m \pi_{ij} = \mu_i,\forall i,~~
\sum_{i=1}^n \pi_{ij} = \nu_j,\forall j,~~
\sum_{i=1}^n \mu_i = \sum_{j=1}^m \nu_j.
\end{align*}
The first constraint says that all mass at each~$\mb p_i$ is appropriately transported; the second constraint is that the demand at each~$\mb r_j$ is met; while the last constraint is mass conservation. 

\begin{figure*}[!h]
  \centering
  \includegraphics[width=1.7in]{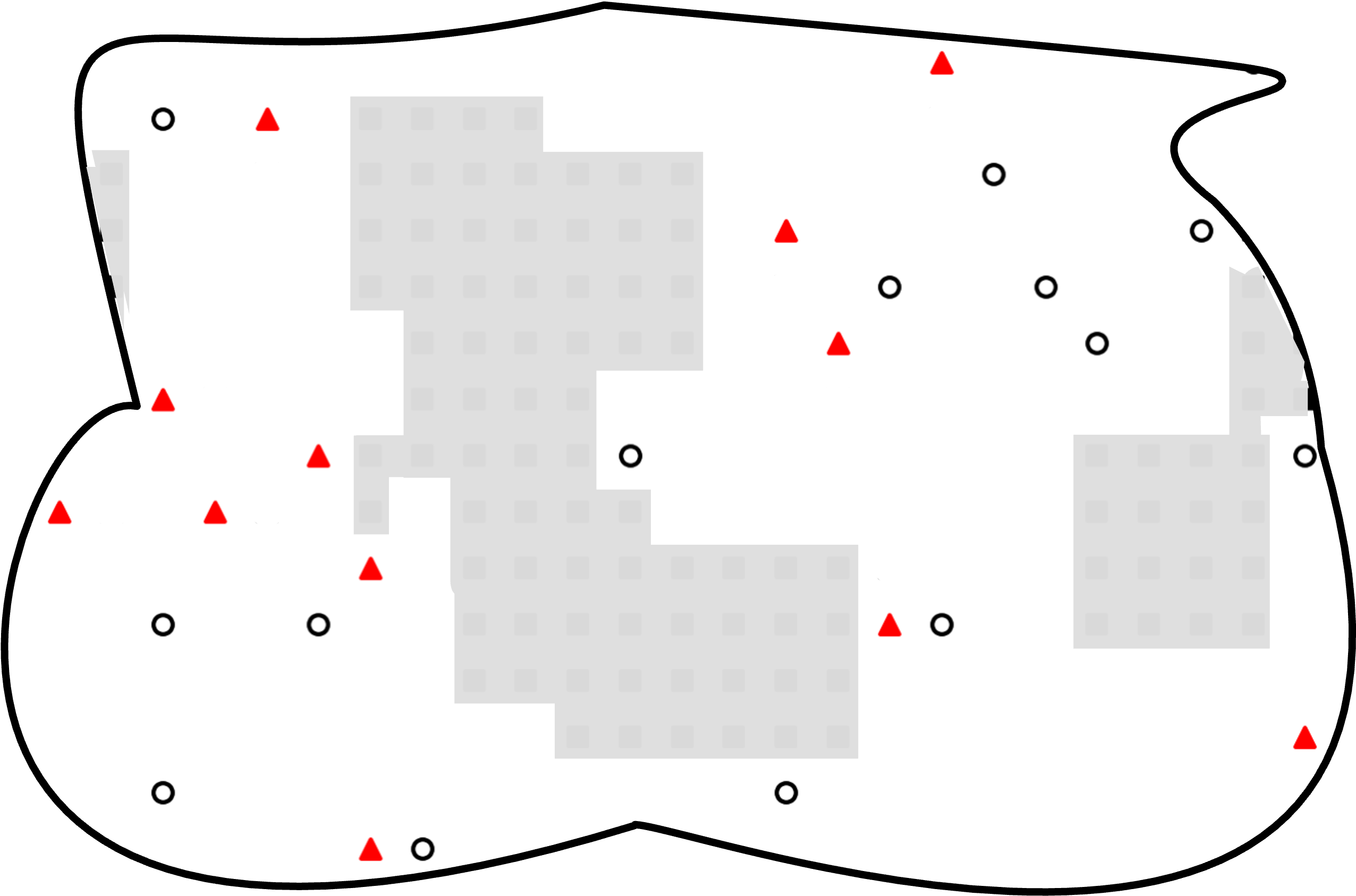}
  \hspace{0.2cm}
  \includegraphics[width=1.7in]{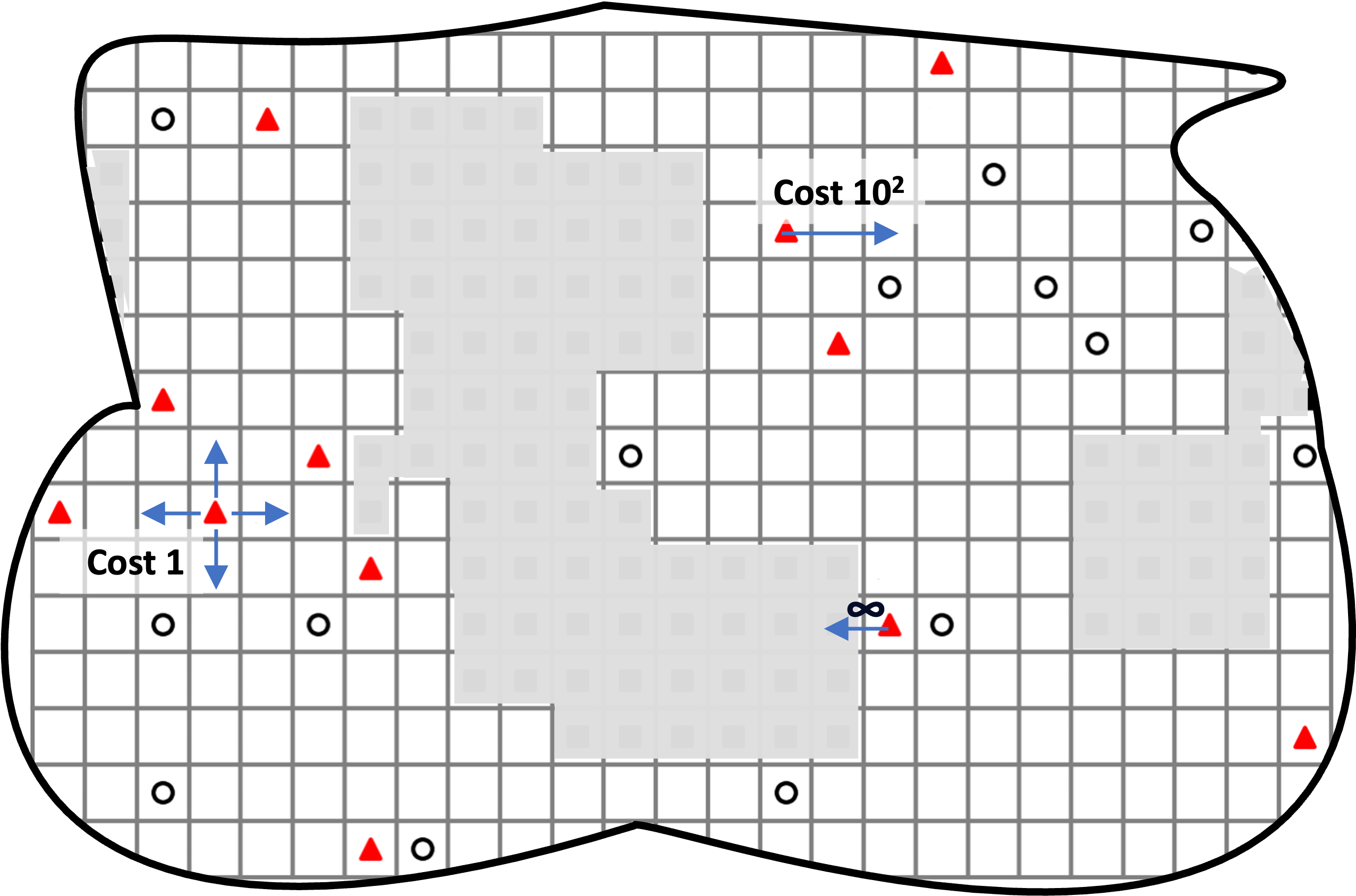}
  \hspace{0.2cm}
  \includegraphics[width=1.7in]{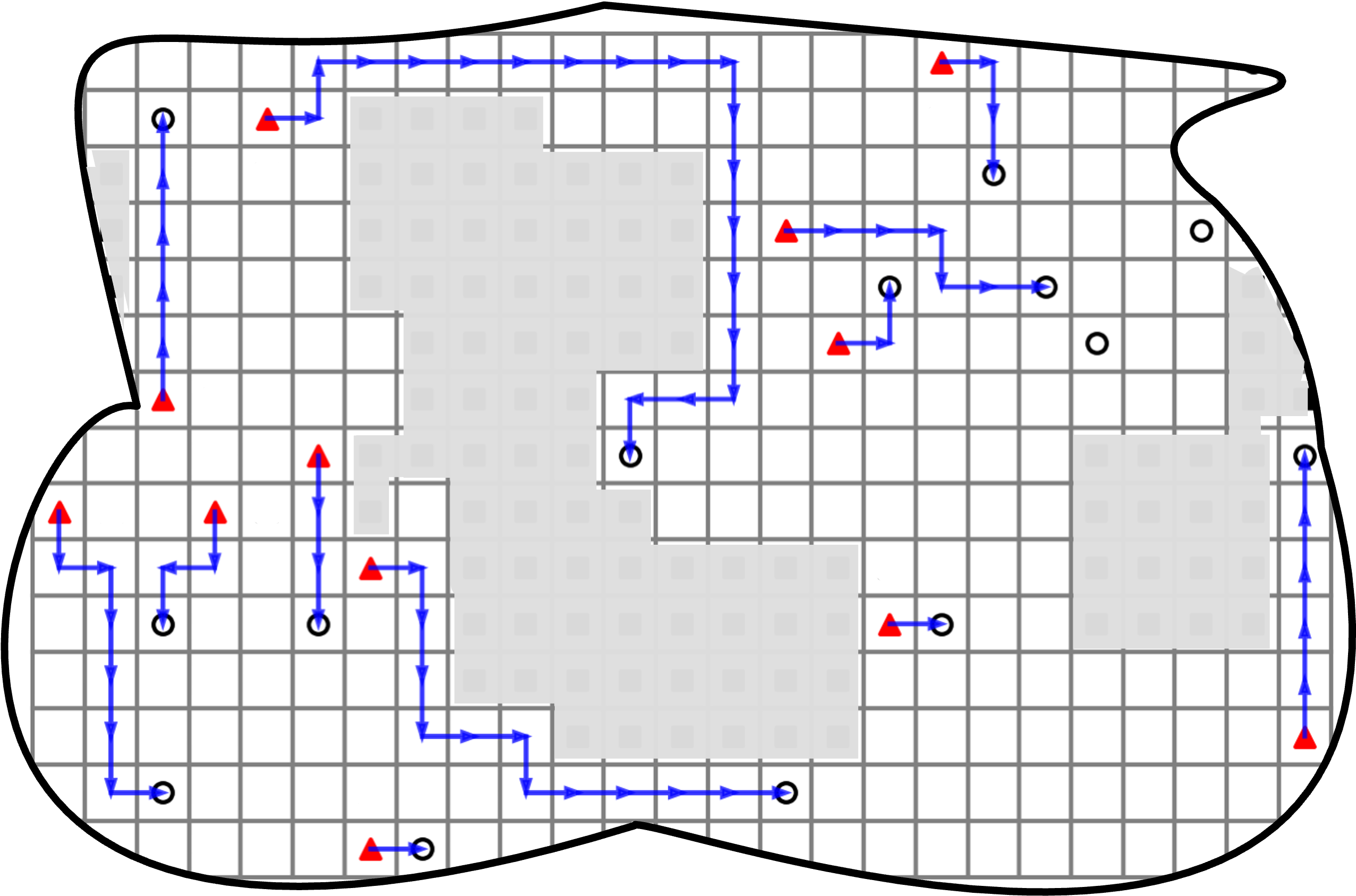}
  \caption{(Left) Robots~$\blacktriangle$ and targets~$\circ$ in the space~$\Omega$ with obstacles in shaded gray: There are~$12$ robots and~$14$ targets indicating that~$2$ targets will be left unassigned. The robots are tasked to travel to distinct targets in a cost optimized way. (Middle) Discretization of~$\Omega$ with the cost of cell transitions. (Right) Minimum-cost paths robots take to targets: Clearly, robots picking nearest target is sub-optimal. }
  \label{fig1}
\end{figure*}
\section{OT-BASED MATCHING and PATH PLANNING}\label{sec_ot}
In this section, we describe the main idea of this paper: to use optimal transport (OT) for joint matching and path planning. The discrete OT formulation described earlier has a direct relation to the robot-target matching problem. In its simplest form, the transport plan~$\Pi$ can be chosen as an~${N\times M}$ matrix with elements in~$\{0,1\}$ such that each row and each column sums to~$1$; giving us the minimum cost robot-target matching. However, this implementation does not consider all available paths from the robots to the targets, does not explicitly account for obstacles, and therefore the resulting robot-target pairings may result in collisions, bottlenecks, and deadlocks. \textit{Our goal is not only to match robots to targets but to also compute optimal, non-overlapping paths that each robot could navigate.} To achieve this, we cast a novel and unique OT formulation on a discretization of the underlying space~$\Omega$. 

\textbf{Discretization of~$\Omega$:}
Let~$\mc D(\Omega)$ be a discretization of~$\Omega$ into~$K$ cells; the cell size/placement is such that no robot or target can occupy more than one cell. An example configuration of~$\Omega$ and the grid-discretization is shown in Fig.~\ref{fig1}; we choose a uniform, equally-spaced discretization for the sake of simplicity, which can be relaxed as we discuss later in Section~\ref{disc_rev}. 

\textbf{Cost Matrix~$C$:} We now design a~${K\times K}$ cost matrix~${C=\{c_{ij}\}}$ that describes the transition cost from one cell to another such that moving to the adjacent cells is cheaper and incurs a cost of, e.g.,~$1$, while moving to any other distant cell has a higher cost and staying in the same cell has~$0$ cost. Similarly, the cost to move into an obstacle is chosen arbitrarily high. An example cost is shown in Fig.~\ref{fig1} (middle). With the help of this space discretization and cost structure, a particular novelty of our approach is that the robots are incentivized to travel to the targets by successively sliding into nearby cells instead of making a long-distance one-step match (jump) to the targets. The following OT formulation formally describes our approach. 

\textbf{OT Formulation:} Optimal transport provides the optimal strategy to move mass from a source distribution to a target distribution. Let~${\boldsymbol{\mu}=\{\mu_k\}\in\mbb R^K}$ be a~$\{0,1\}$ vector of the \textit{source distribution} such that
\begin{align}\label{mueq}
\mu_k = \left\{
\begin{array}{cc}
    1, & \mbox{cell~$k$ is not a target},\\
    0, & \mbox{cell~$k$ is a target}.
\end{array}\right.
\end{align}
Similarly, let~${\boldsymbol{\nu}=\{\nu_k\}\in\mbb R^K}$ be a~$\{0,1\}$ vector of the \textit{target distribution} such that
\begin{align}\label{nueq}
\nu_k = \left\{
\begin{array}{cc}
    1, & \mbox{cell~$k$ is not a robot},\\
    0, & \mbox{cell~$k$ is a robot}.
\end{array}\right.
\end{align}
This unconventional encoding of the source and target distributions is then used in OT to move the mass at~$\bmu$ to~$\bnu$. Note that, by design, both~$\bmu$ and~$\bnu$ include a \textit{virtual} one-unit mass on the \textit{entire} free space where we have neither robots nor targets. Since the cost of staying in place is zero, this virtual mass is displaced only if it facilitates a lower-cost sequence of transitions to transport the non-virtual mass in the source (robots) to the non-virtual mass in the target distribution. An alternate viewpoint of this mass transport is that, instead of matching the robots to targets (and solving an~${N\times M}$ assignment), the proposed setup assigns mass to the entire space~$\Omega$ such that mass transport is achieved by reconfiguring~$\Omega$ from a source configuration to a desired target configuration in a cost-optimal way. The following example explains this concept. 
\begin{figure}[!htb]
  \centering
  \includegraphics[width=3in]{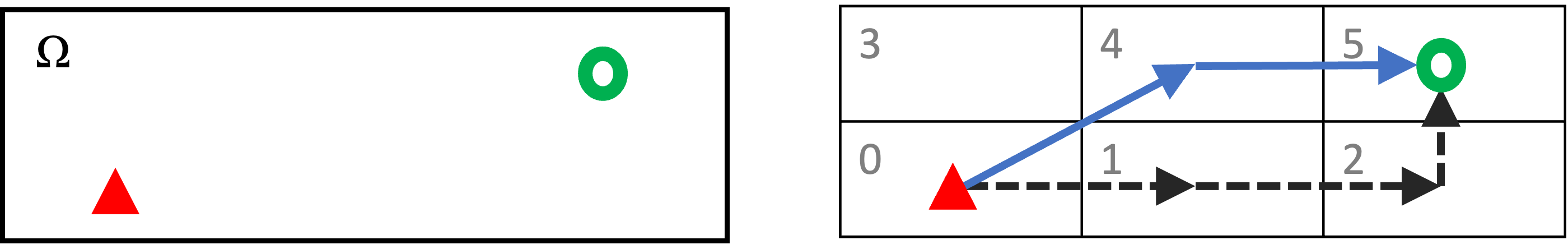}
  \caption{(Left) Example scenario. (Right) Discretization and viable paths. }
  \label{fig2}
\end{figure}

\textit{Example: }Consider a simple one robot and one target case in a region of interest~$\Omega$ (Fig.~\ref{fig2}). We discretize~$\Omega$ into a~$3\times 2$ grid, with~${K=6}$ cells. With this discretization, the robot occupies cell~$0$, while the target occupies cell~$5$. An example cost matrix for this problem is
\begin{align*}
    C = \left[
    \begin{array}{cccccc}
         0 & 1 & 10^2 & 1 & 1 & 10^3\\
         1 & 0 & 1 & 1 & 1 & 1\\
         10^2 & 1 & 0 & 10^3 & 1 & 1\\
         1 & 1 & 10^3 & 0 & 1 & 10^2\\
         1 & 1 & 1 & 1 & 0 & 1\\
         10^3 & 1 & 1 & 10^2 & 1 & 0
    \end{array}
    \right],
\end{align*}
where~$c_{ij}$ is the cost of moving mass from cell~$i$ to cell~$j$.  The source and target distribution vectors are~${\bmu^\top = [1~1~1~1~1~0]}$ and ${\bnu^\top = [0~1~1~1~1~1]}$, respectively. The assignment matrix~${\Pi=\{\pi_{ij}\}}$, shown below, is such that~$\pi_{ij}$ represents how much of the mass~$\mu_i$, at location~$i$, is transported to meet the demand~$\nu_j$, at location~$j$. In other words, we would like the row sum of~$\Pi$ to match~$\bmu$ and the column sum of~$\Pi$ to match~$\bnu$ ensuring that the source mass is successfully transported to the target. Since the cost of keeping mass at its location is zero, i.e.,~${c_{i=j}=0}$, choosing~$\Pi$ to be identity could result in the smallest (zero) cost. However, since no mass can be brought to cell~$0$ (column-sum constraint), we must have~${\pi_{00}=0}$, i.e., mass at cell~$0$ cannot stay in place. Consequently, mass at cell~$0$ must incur a non-zero transition cost to ensure that it is transported out (row-sum constraint). In other words, the robot must be displaced to an adjacent cell~$1,3$, or~$4$ (each with a unit cost). Suppose,~${\pi_{04}=1}$, i.e., the robot slides into cell~$4$ leading to~${\pi_{44}\neq1}$ (column-sum constraint); mass at cell~$4$ therefore moves out to satisfy the row-sum constraint, which can be achieved by~${\pi_{45}=1}$. The rest of the assignment is~${\pi_{11}=\pi_{22}=\pi_{33}=1}$, each with zero cost. The entire assignment is shown by underlined variables below and can be verified that it matches the constraints. The non-zero costs describe the path~$0\ra4\ra5$ taken by the robot to the target. 
\begin{align*}
    &\begin{array}{cccccc}
    \hspace{-0.75cm}\bmu\backslash\bnu~~\:\:0 & ~1 & ~~\:1 & \!~~~1 & ~~~1 & ~~1
    \end{array}\\
    \left.
    \begin{array}{c}
         1\\
         1\\
         1\\
         1\\
         1\\
         0
    \end{array}\right|
    &\left[
    \begin{array}{cccccc}
         0 & \pi_{01} & \pi_{02} & \pi_{03} & \underline{\pi_{04}} & \pi_{05}\\
         0 & \underline{\pi_{11}} & \pi_{12} & \pi_{13} & \pi_{14} & \pi_{15}\\
         0 & \pi_{21} & \underline{\pi_{22}} & \pi_{23} & \pi_{24} & \pi_{25}\\
         0 & \pi_{31} & \pi_{32} & \underline{\pi_{33}} & \pi_{34} & \pi_{35}\\
         0 & \pi_{41} & \pi_{42} & \pi_{43} & \pi_{44} & \underline{\pi_{45}}\\
         0 & 0 & 0 & 0 & 0 & 0
    \end{array}
    \right]
\end{align*}

\subsection{Balanced Case~$(N=M)$} 
The above exposition can be cast in a precise mathematical framework. In particular, the OT problem is to find a transport plan~$\Pi$, among all possible plans, that results into the minimum cost such that the source vector~$\bmu$ in~\eqref{mueq} is mapped to the target vector~$\bnu$ in~\eqref{nueq}, i.e., 
\begin{align*}
\mbox{\textbf{P1:}~~~} \inf_{\Pi=\{\pi_{ij}\}} \sum_{i}&\sum_j \pi_{ij}\cdot c_{ij},\qquad i,j=\{1,\ldots,K\}, \\
\mbox{s.t.~~~~}&\Pi\cdot \mb1_K=\bmu,~~
\mb 1_K^\top \Pi = \bnu^\top,
\end{align*}
where~$\mb 1_K$ is a column vector with~$K$ ones; the constraints ensure that the transport has the correct marginals. Clearly, multiple solutions may exist and the minimum cost path is not necessarily unique (both ${0\ra4\ra5}$ and ${0\ra1\ra5}$ in Fig.~\ref{fig2}). Optimal transport provides one such min-cost transport as noted by the~$\inf$ in~\textbf{P1}. To describe the main results and guarantees of~\textbf{P1}, we next provide a key concept. 

\begin{defn}\label{pfdisc}
A practically feasible discretization~$\overline{\mc D}(\Omega)$ of~$\Omega$, with~$K$ cells, is such that there exists a transport plan~$\Pi$ on~$\overline{\mc D}(\Omega)$ with the following properties: 
\begin{enumerate}[(i)]
\item $\mb1_K^\top (\Pi\odot C)\mb 1_K$ is finite; 
\item the row- and column-sum constraints are satisfied by~$\Pi$;
\item if~$\pi_{ij}=1$, for any~$i$ and~$j$, then~$c_{ij}\leq 1$.
\end{enumerate}
\end{defn}

Definition~\ref{pfdisc} formalizes all possible discretizations of~$\Omega$, where the OT problem has a well-behaved solution, achievable by at-most one-step transitions (we assume for simplicity that one-step transitions that are physically realizable have a cost of~$\leq1$). In other words,~$\Pi$ does not require a cell transition where a robot jumps more than one cell. Scenarios that violate Def.~\ref{pfdisc} will be considered in the next section. Key implications and guarantees regarding \textbf{P1} are presented next. 

\begin{lem}\label{lem1}
    Let~${N=M}$ robots and targets be situated~in~$\Omega$ with a practically feasible discretization~$\overline{\mc D}(\Omega)$. The transport~$\Pi^*$ obtained from the OT Problem in \textbf{P1}, described on~$\overline{\mc D}(\Omega)$, has the following properties: 
    \begin{enumerate}[(i)]
    \item $\Pi^*$ has~$\{0,1\}$ elements, i.e., the robots do not split on their paths to the targets; 
    \item $\Pi^*$ provides non-overlapping paths; 
    \item The total mass is conserved by the transport, i.e., all robots reach a distinct target;
    \item $\Pi^*$ is optimal under~$\overline{\mc D}(\Omega)$ and the OT constraints;
    \item The complexity of \textbf{P1} ranges from~$\mc O(K^3\log K)$ to~$\mc O(K^2)$ for well-behaved problems. 
    \end{enumerate}
\end{lem}
\begin{proof}
    \begin{inparaenum}[(i)]
    It can be verified that \textbf{P1} is a linear program (LP) with linear constraints and is well-defined over~$\overline{\mc D}(\Omega)$, leading to the following. 
    \item The marginal constraints result in a constraint matrix that is \textit{totally unimodular}, ensuring that the underlying LP has integral solutions;
    \item We can prove this by contradiction. Assume that two paths intersect at cell~$j$, then masses from two different cells~$i_1$ and~$i_2$ arrive at cell~$j$, i.e.,~${\pi_{i_1j}=\pi_{i_2j}=1}$, which violates the column-sum constraint since~${\bmu,\bnu,\Pi\in\{0,1\}}$. The same argument applies to other situations. 
    \item We have~${\boldsymbol{\mu}^\top\mb1_K=\boldsymbol{\nu}^\top\mb1_K}$, since~${N=M}$, and the mass conservation constraint is satisfied, which in addition to~(i)-(ii) ensures that all robots reach a distinct target; 
    \item \textbf{P1} is a feasible LP and therefore finds one of the optimal solutions. 
    \item LPs can be solved using interior point methods, at~$\mc O(K^3 \log K)$, or with auction algorithms at~$\mc O(K^2)$ under favorable conditions.
    \end{inparaenum}  
    \qed
\end{proof}
Assertion (i) represents a key theoretical contribution of this work, i.e., it establishes that the proposed OT formulation inherently preserves integrality, ensuring that mass is never split during the entire transport. This result leverages classical unimodularity results for integral solutions to LPs~\cite{schrijver1998}, but its application in this setting is both novel and essential. Assertions (ii)–(iv) then build on this foundation and ensure non-overlapping, optimal paths that cover all robots and targets, while assertion (v) follows from standard arguments on algorithmic and computational aspects~\cite{Dantzig1968}.

\subsection{Unbalance Case~$(N\neq M$)}
The classical OT setup described above in \textbf{P1} returns a transport plan~$\Pi$ that preserves the total mass between the source and target, ensuring that each robot reaches a distinct target. In scenarios with unequal robots and targets, the formulation in \textbf{P1} can be modified with the help of unbalanced optimal transport (UOT) as described next~\cite{Chizat2018,Chapel2020PartialOT}. Let~${m=\min(\|\bmu\|_1,\|\bnu\|_1)}$ be the least amount of mass that must be transported. The UOT formulation is given by
\begin{align*}
\mbox{\textbf{P2:}~~~} &\inf_{\Pi=\{\pi_{ij}\}} \sum_{i}\sum_j \pi_{ij}\cdot c_{ij},\qquad i,j=\{1,\ldots,K\}, \\
\mbox{s.t.~~}&\Pi\cdot \mb1_K\leq\bmu,~~
\mb 1_K^\top \Pi \leq \bnu^\top,~~\mb 1_K^\top \Pi\mb 1_K= m.
\end{align*}
This formulation is best explained with a simple example. Consider the same setup as in Fig.~$2$ but with an additional robot at cell~$1$. The cost matrix remains the same as before. The transport plan~$\Pi$ with the UOT constraints in \textbf{P2} is 
\begin{align*}
    &\begin{array}{cccccc}
        \hspace{-0.75cm}\bmu\backslash\bnu~~\:\:0 & 0 & ~~1 & ~~1 & ~~1 & ~~~1
    \end{array}\\
    \left.
    \begin{array}{c}
         1\\
         1\\
         1\\
         1\\
         1\\
         0
    \end{array}\right|
    &\left[
    \begin{array}{cccccc}
         0 & 0 & \pi_{02} & \pi_{03} & {\pi_{04}} & \pi_{05}\\
         0 & 0 & \pi_{12} & \pi_{13} & \pi_{14} & \pi_{15}\\
         0 & 0 & {\pi_{22}} & \pi_{23} & \pi_{24} & \pi_{25}\\
         0 & 0 & \pi_{32} & {\pi_{33}} & \pi_{34} & \pi_{35}\\
         0 & 0 & \pi_{42} & \pi_{43} & \pi_{44} & {\pi_{45}}\\
         0 & 0 & 0 & 0 & 0 & 0
    \end{array}
    \right]
\end{align*}
where the additional robot appears as~${\nu_{i1}=0},\forall i$. Following the column-sum constraint, the first two columns are now zeros implying that mass cannot be brought to either cell~$0$ or cell~$1$. The UOT inequality constraints allow for a zero solution, which is avoided by the final constraint that~$\Pi$ must have exactly~${m=\min(5,4)=4}$ ones. Among several possible plans~$\Pi$, two examples are~$\{{0\ra4\ra5}, {2\ra2},{3\ra3}\}$ and~$\{{1\ra5},{2\ra2},{3\ra3},{4\ra4}\}$, where the latter has a cost of~$1$ and is returned by UOT. 

\section{REPLANS with OT}\label{sec_replans}
The OT and UOT formulations discussed up to now assume a \textit{practically feasible} discretization~$\overline{\mc D}(\Omega)$. As a consequence, the linear program implemented by OT is guaranteed to find a min-cost solution, with non-overlapping paths that comprise of at-most one cell transitions. We now consider scenarios where a preferred discretization~$\underline{\mc D}(\Omega)$, not necessarily feasible, is given and the goal is to design robot trajectories within the confines of~$\underline{\mc D}(\Omega)$. For a given~$\underline{\mc D}(\Omega)$, however, implementing trajectories obtained from OT may not be ideal, see Fig.~\ref{fig3}:
\begin{inparaenum}[(i)]
    \item $\underline{\mc D}(\Omega)$ is practically feasible (Def.~\ref{pfdisc}) but results in a longer trajectory;
    \item non-overlapping trajectories with at most one-cell transitions do not exist.
\end{inparaenum}
These are consequences of the fact that the OT plans do not account for timing and scheduling possibilities. We emphasize that OT plans are still meaningful when~${\underline{\mc D}(\Omega)=\overline{\mc D}(\Omega)}$, as they can be implemented without any temporal considerations, as in Fig.~\ref{fig3} (left). 
\begin{figure}[!h]
  \centering
  \includegraphics[width=3in]{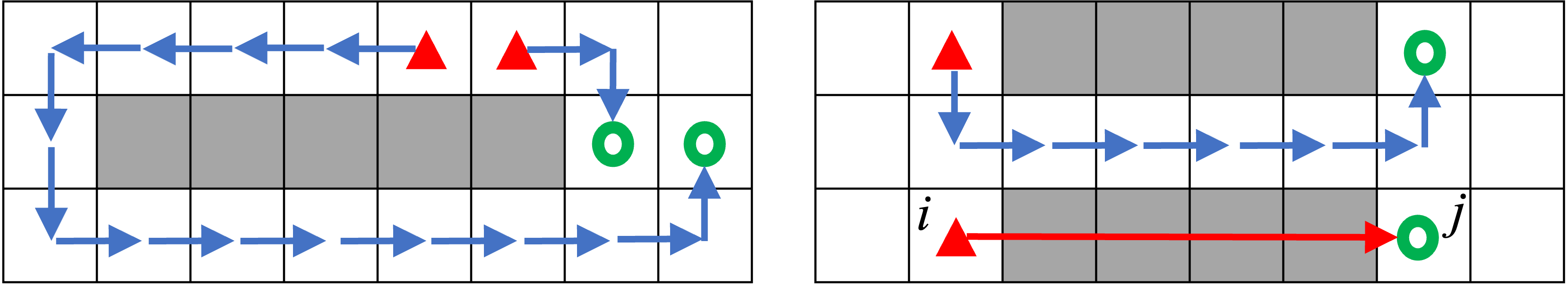}
  \caption{OT trajectories on a given~$\underline{\mc D}(\Omega)$: (Left) Longer trajectory to satisfy constraints; (Right) Non-overlapping one-step trajectories do not exist and OT results in a jump solution with a total cost of~${7+c_{16}},{1\ll{c_{16}<\infty}}$.}
  \label{fig3}
\end{figure}

\textbf{Simple replans: }One approach to address potentially longer (but non-overlapping) OT trajectories is to sort them in an ascending-cost order and execute the shorter paths first (composed by unit cost transitions). The OT problem is then re-solved with a reduced set of robots and targets. Clearly, this approach handles both scenarios in Fig.~\ref{fig3}, without bringing in path conflicts. For example, the shorter path in Fig.~\ref{fig3} (left) is executed first in~$2$ steps and then OT is solved again to find the shortest path for the remaining robot. Another approach is to find a set of overlapping trajectories and bring timing and scheduling considerations to avoid conflicts. We describe this next. 

\subsection{Discretization (revisited)}\label{disc_rev}
Consider the scenario in Fig.~\ref{fig3} again with a given~$\underline{\mc D}(\Omega)$. Let the robots and targets be point masses and assume a finer discretization~$\underline{\underline{\mc D}}(\Omega)$ of~$\underline{\mc D}(\Omega)$, such that~$\underline{\underline{\mc D}}(\Omega)$ is practically feasible, i.e., OT returns lower cost on a higher resolution grid; see Fig.~\ref{fig4}. In other words, the high resolution in~$\underline{\underline{\mc D}}(\Omega)$ is equivalent to increasing the travel capacity of the corresponding cells. Clearly, from the perspective of OT in \textbf{P1} or \textbf{P2}, the solutions in Fig.~\ref{fig4} are feasible as no constraints are violated because the trajectories traverse distinct spaces. These trajectories can be projected back in~$\underline{\mc D}(\Omega)$, where a transport plan is originally sought, however, they cannot be readily implemented and require scheduling considerations with the help of replans as discussed in the next section. The following lemma characterizes the finer discretizations~$\underline{\underline{\mc D}}(\Omega)$.
\begin{figure}[!h]
  \centering
  \includegraphics[width=3in]{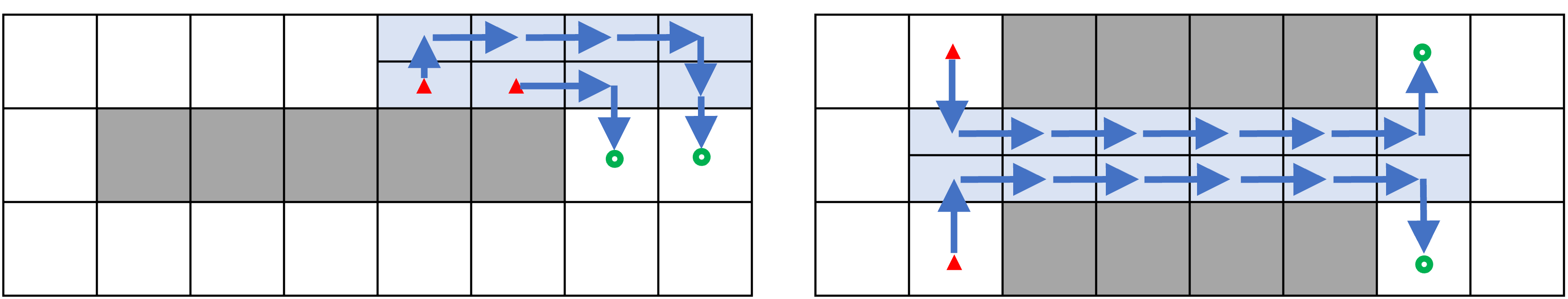}
  \caption{Non-overlapping paths by finely discretizing~$\Omega$.}
  \label{fig4}
\end{figure}



\begin{lem}\label{lem_valid}
    Consider~${N}$ robots and~$M$ targets in~$\Omega$, ${N,M<\infty}$, with a preferred discretization $\underline{\mc D}(\Omega)$. Assume that each robot has a path (in $\Omega$) to each target that does not cross through an obstacle, another robot, or another target.
    Then there exists a finer discretization~$\underline{\underline{\mc D}}(\Omega)$ of~$\underline{\mc D}(\Omega)$, which is practically feasible.
\end{lem}
A formal proof is beyond the scope of this paper, but we provide a brief sketch of the argument. If each robot has at least one unblocked path to a target, feasible paths can be accommodated even in narrow regions by using finer discretizations to resolve conflicts. The condition stated in Lemma~\ref{lem_valid} avoids situations where a robot or a target permanently blocks another robot from reaching its goal, for example, by occupying the only traversable corridor. This requirement is related to the concept of \textit{valid infrastructures} introduced in~\cite{Cap2015}, which ensure that agents can maneuver without creating deadlocks. We note that non-uniform or adaptive discretizations, using finer resolution only in tight or congested areas, are also possible; see Fig.~\ref{fig4}. While finer discretizations increase the likelihood of feasible, schedulable paths, they also incur greater computation, both in computing OT and in generating and implementing schedules. However, the OT complexity remains~$\mc O(K^3 \log K)$, provided the number of subdivisions does not scale with~$K$. 

In the next section, we use the set of paths generated by OT and design MPC strategies for each robot to track these paths while respecting the robot dynamics and control costs. The proper integration of OT replans with a receding horizon strategy allows us to extend the OT planning to the time-varying case, e.g., when the  obstacles are dynamic. 

\section{MPC-OT: TRAJECTORY TRACKING}\label{sec_mpc}
Up to now, we have neglected the robots' dynamics in the OT formulations, however, in most real-life applications the path planned by OT will be executed by robots with physical constraints, given by (\ref{ss1}). To include these dynamics in the overall multi-robot system control, we propose to close the loop on these trajectories via decentralized model predictive control (MPC)~\cite{Rossiter2004,Benosman2016}. To this aim, for each robot~${n\in[1,\ldots,N]}$, let~$\mb y_n^d(r,t)$ be the desired trajectory generated by the~$r$th OT replan, obtained by breaking the OT paths, resulting for example from the aforementioned simple replans, into smaller segments. Let~$[t_k,t_{k+1}]$ denotes the time interval horizon of the MPC applied to the system, such that~${t_k =t}$,~${t_{k+1}=t+T_c}$,~${T_{c}<T}$, with $T>0$ denoting the MPC integration horizon. For each robot~$n$, we solve the following contractive MPC problem: 
\begin{equation}\label{MPC-pb}
\begin{array}{l}\min\limits_{\mb u_n }J_{n}:=\int_{t}^{t+T}(\mb y_n(t)-\mb y^{d}_n (r,t))^\top Q_{1}(\mb y_n(t)-\mb y^{d}_n(r,t))dt\\\hspace{+1.5cm}+\int_{t}^{t+T}\mb u_{n}(t)^{\top}Q_{2}\mb u_{n}(t)dt,\end{array}
\end{equation}
under the constraints
\begin{equation}\label{MPC-pb-cts}
\begin{array}{l}
\dot {\mb x}_{n}(t)=f_{n}(\mb x_{n}(t))+g_{n}(\mb x_{n}(t))\cdot\mb u_{n}(t)\\
\mb y_{n}(t)=h_{n} (\mb x_{n}(t))\\
\mb u_{n}^{-}\leq \mb u_{n}(t)\leq \mb u_{n}^{+}\\
\mb x_{n}^{-}\leq \mb x_{n}(t)\leq
\mb x_{n}^{+}\\
||\mb e_{n}(t_{k+1})||_P\leq \alpha_{n}
||\mb e_{n}(t_{k})||_P,~~\alpha_n\in(0,1),
\end{array}
\end{equation}
where~${\mb e_{n}(t)=h_n(\mb x_n(t))-\mb y^{d}_{n}(r,t)}$ is the error process in tracking the trajectory, the matrices~$Q_1,Q_2,P$ are arbitrary, positive definite, and~$||\mb e_{n}(t)||_P=(\mb e_{n}(t)^{\top}P\mb e_{n}(t))^{1/2}$. Note that the first two constraints are the physical dynamics~\eqref{ss1}, the third and fourth constraints enforce that the state and control remains inside a bounded interval, while the last constraint ensures that the error process~$\mb e_n(t)$ is contractive. 

We will now prove boundedness of the tracking errors and exponential convergence to the target locations for all robots, when controlled by the MPC-OT  presented in Algorithm \ref{alg:MPC-OT}. Before we proceed, we provide the assumptions next.

\begin{algorithm}
\caption{MPC-OT PSEUDO-CODE}\label{alg:MPC-OT}
\begin{algorithmic}[1]
\renewcommand{\algorithmicrequire}{\textbf{Input:}}
\REQUIRE Number of robots~$N$,\\ 
~~~~\:Initial robot locations $y_{n}(0)$, ${n=1,\ldots,N}$,\\ 
~~~~\:MPC cost horizon~$T$,\\
~~~~\:MPC control horizon~$T_c$,\\
~~~~\:Target robot locations~$y^{target}_{n}(T)$, $n=1,\ldots,N$,\\ 
~~~~\:OT cost matrix $C$,\\
~~~~\:MPC time-step discretization~$\delta t$\;
\STATE Start OT loop: encode $y_{i}(0)$'s in~$\bnu$, and $y^{target}_{i}$'s in~$\bmu$
\STATE Solve the OT problem $\bf{P1}$ or $\bf{P2}$ with cost~$C$\;
\STATE Start the MPC loop: generate a trajectory from the discrete OT optimal path (any time-interpolation can be used here)\; 
\STATE Solve the MPC problem (\ref{MPC-pb}) over horizon~$T$\;
\STATE Apply the control inputs $u_{n}$,$n=1,\ldots,N$ for the time interval~$[t,t+T_{c}]$\;
\STATE Reset $y_{i}(0)$'s to $y_{i}(t+T_{c})$'s, and go to Step 1: (OT replanning)\;
\end{algorithmic}
\end{algorithm}

\begin{assumption}\label{assump1}
For any ${n\in [1,\ldots,N]}$, the system in (\ref{ss1}) is controllable along its OT target output trajectories.
\end{assumption}
\begin{assumption}\label{assump2}
For any ${n\in [1,\ldots,N]}$ and any $r\in \{1,2,\ldots\}$, there exists a
$\rho_{nr}\in\:]0,\infty[$ such that for all
${\mb x_{t_{0}}\in\mb{B}_{\rho_{nr}}}\stackrel{\triangle}{=}\{{\mb x_{t_{0}}\in\mathbb{R}^{n}}|\;||\mb e_{n}(\mb x_{t_{0}})||_{P}
\leq\rho_{nr}\}$, the MPC problem (\ref{MPC-pb}), (\ref{MPC-pb-cts}) has a
solution. 
\end{assumption}

\begin{assumption}\label{assump3}
For any ${n\in [1,\ldots,N]}$ and any $r\in \{1,2,\ldots\}$, there exist a constant ${\beta_{nr}\in\:]0,\infty[}$, such that
$||\mb e_{n}(t)||_{P}\leq\beta_{nr}||\mb e_{n}(t_{k})||_{P},\;\forall
t\in[t_{k},\;t_{k+1}],\;k=0,1,\dots$.
\end{assumption}

\begin{assumption}\label{assump4}
For any ${n\in [1,\ldots,N]}$, and any $r\in \{1,2,\ldots\}$, there exists a $(t^{*},r^{*})$, s.t.,~$\forall {t>t^{*}},\;{r>r^{*}}$ the OT generates the same optimal trajectories $\mb y^{d^{*}}_{n}(r^{*},.)$, i.e., the OT replanning leads to the same target trajectories. 
\end{assumption}

Assumptions~\ref{assump1} and~\ref{assump2} are classical in MPC, allowing us to ensure the feasibility of the MPC problem. The third assumption on the solutions' boundedness is straightforward from the box constraints on the inputs and the continuity of the system’s model. The fourth assumption is needed to allow the system to \textit{settle} on a final optimal trajectory to track. This assumption is needed to prove exponential stability of the error dynamics w.r.t. to the final optimal trajectory.

\begin{thm}\label{th1}
Under Assumptions \ref{assump1}--\ref{assump4}, the MPC-OT Algorithm~\ref{alg:MPC-OT} applied to the multi-robot system (\ref{ss1}), leads to bounded tracking errors $\mb e_{n}(t)$ and exponential convergence to the target locations $\mb y^{target}_{n}(T),\forall n=[1,\ldots,N]$.
\end{thm}
The proof of Theorem~\ref{th1} will be provided elsewhere. In short, it shows that the robots track, within a bounded error, the trajectories generated by the iterations of the OT planner, and eventually converge exponentially to the final OT trajectory leading them to the desired targets.

\section{Simulations}\label{sec_sims}
In this section, we evaluate the effectiveness of the proposed MPC-OT algorithm with the help of numerical simulations. We consider a region~$\Omega$ discretized into a $75 \times 50$ grid, yielding ${K=3,\!750}$ cells; see Fig.~\ref{fig:mpc}. A team of ${N=20}$ unicycle robots is deployed in this environment, each governed by the following dynamics:
\begin{equation*}
    \begin{aligned}
        \dot p_{x,n} = v_n\cos{\theta_n}, ~~\dot p_{y,n} = v_n\sin{\theta_n}, ~~\dot \theta_n = \omega_n,
    \end{aligned}
\end{equation*}
where $p_{x,n}$ and $p_{y,n}$ denote the position of robot $n$, $\theta_n$ is its orientation, and $v_n$, $\omega_n$ are its forward and angular velocities. The state, control input, and spatial location of each robot are denoted by $\mathbf{x}_n = [p_{x,n}, p_{y,n}, \theta_n]$, $\mathbf{u}_n = [v_n, \omega_n]$, and $\mathbf{y}_n = [p_{x,n}, p_{y,n}]$, respectively. We set the number of targets equal to the number of robots (${M = N}$) and randomly generate obstacles with varying shapes, sizes, and locations. Robots and targets are randomly placed in the obstacle-free space and are marked as red triangles and blue circles, respectively. The optimal transport problem~{\bf P1} is then solved to compute minimum-cost, non-overlapping trajectories that assign each robot to a corresponding target.
\begin{figure}[!h]
    \centering
    \includegraphics[width=2in]{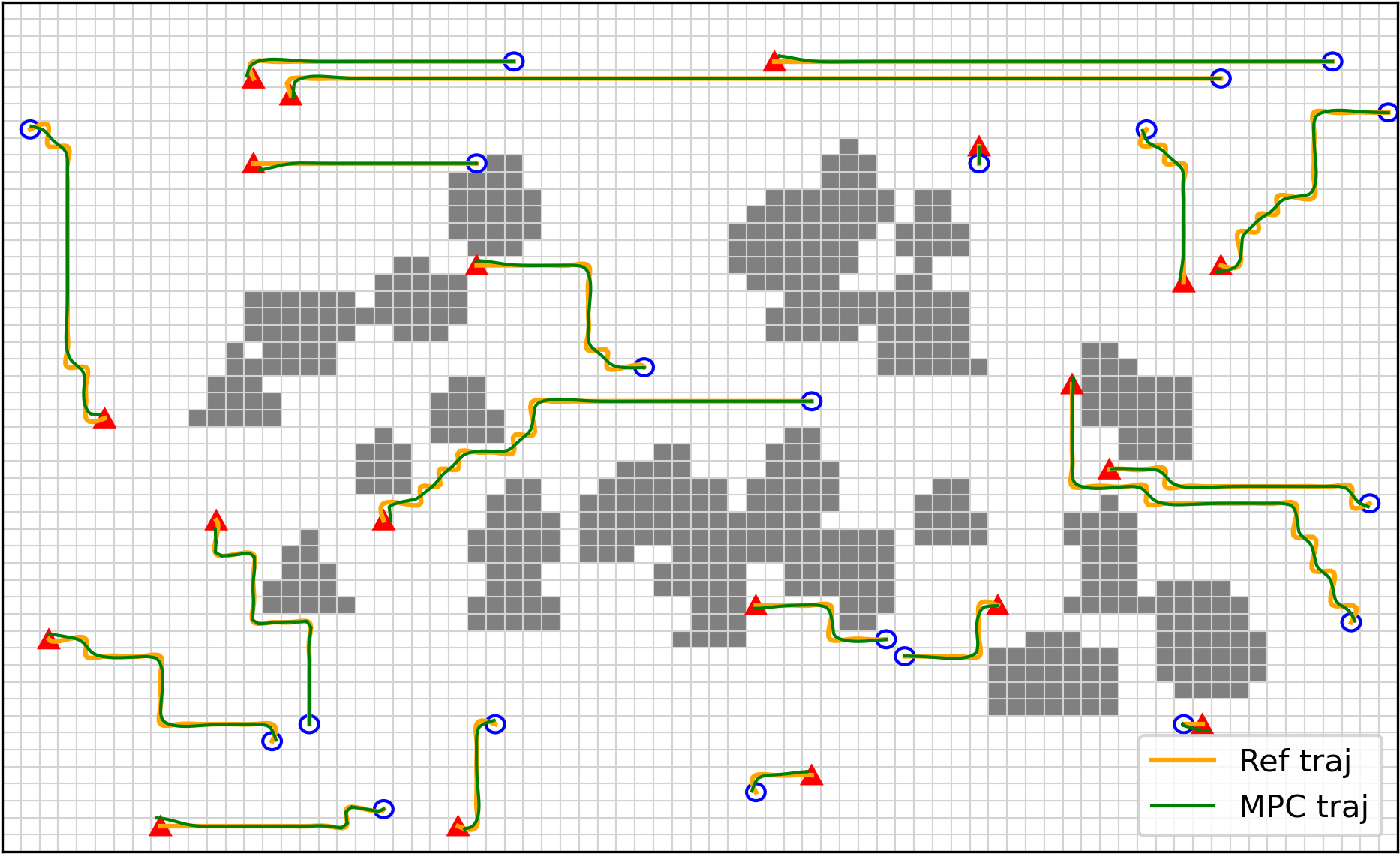}
    \caption{Robot trajectories tracking OT-generated paths: Red triangles indicate initial robot positions; blue circles denote target locations; grey cells represent obstacles. The orange curves are interpolated reference trajectories, while the green curves are the actual MPC-tracked paths.}
    \label{fig:mpc}
\end{figure}

To obtain reference trajectories for model predictive control (MPC), we apply cubic time interpolation to the discrete OT-generated paths, assuming a fixed transition time of $1$ second between neighboring cells. These continuous reference trajectories are tracked using the MPC strategy described in Sec.~\ref{sec_mpc}. Each robot's orientation in its starting cell is aligned with the direction of its motion. As shown in Fig.~\ref{fig:mpc}, all robots successfully follow their assigned OT paths and reach their respective targets. In dynamic environments, the set of occupied cells may change over time due to moving obstacles or new sensor information, necessitating online replans. To simulate such scenarios, we relocate one obstacle during MPC execution, that is effectively captured by the OT replans. As illustrated in Fig.~\ref{fig:replan}, the robot’s reference trajectory is updated in response to the displaced obstacle, and the MPC controller successfully adapts to guide the robot to its target without collision.
\begin{figure}[!h]
    \centering
    \includegraphics[width=2in]{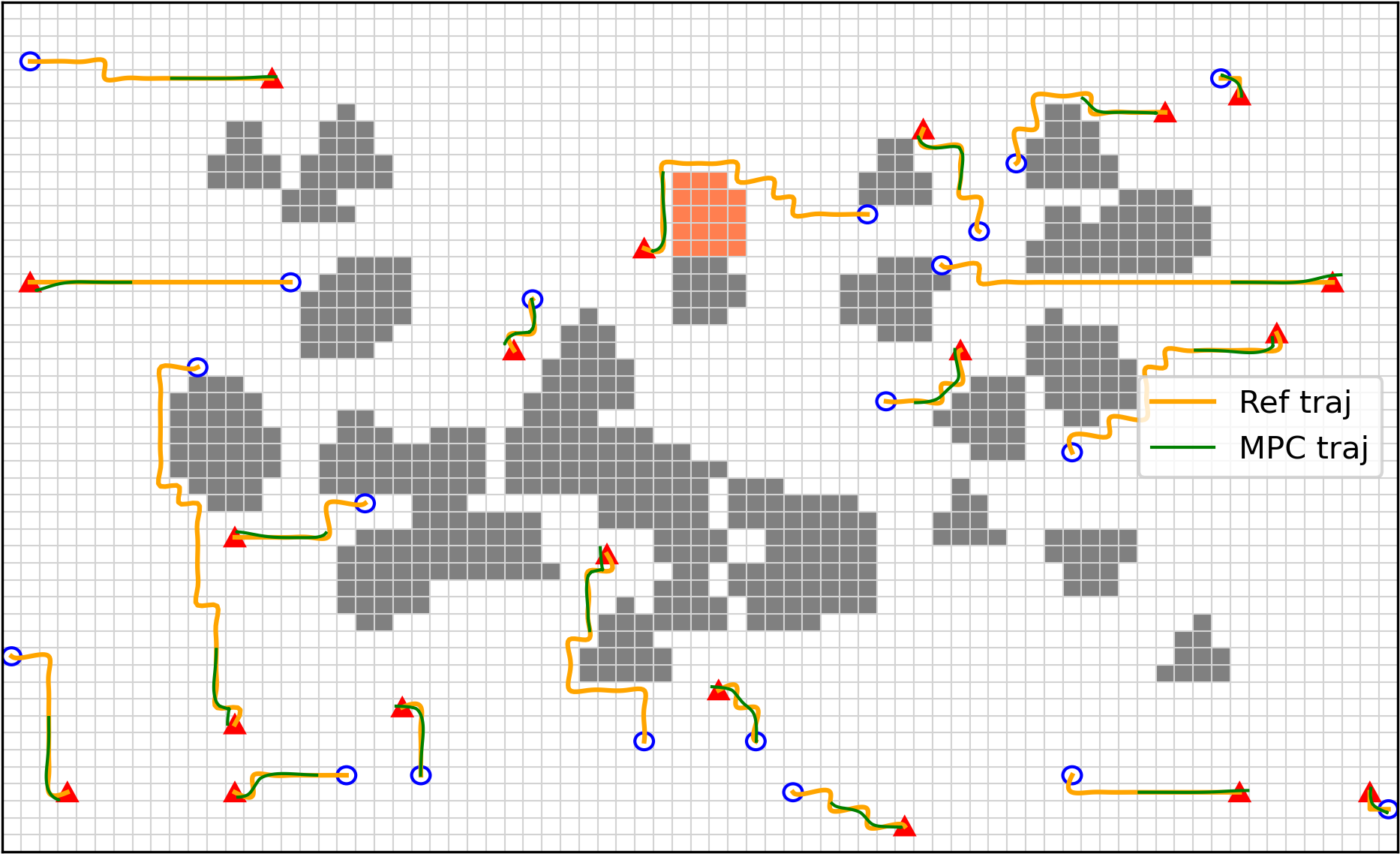}  \hspace{0.1cm}
    \includegraphics[width=2in]{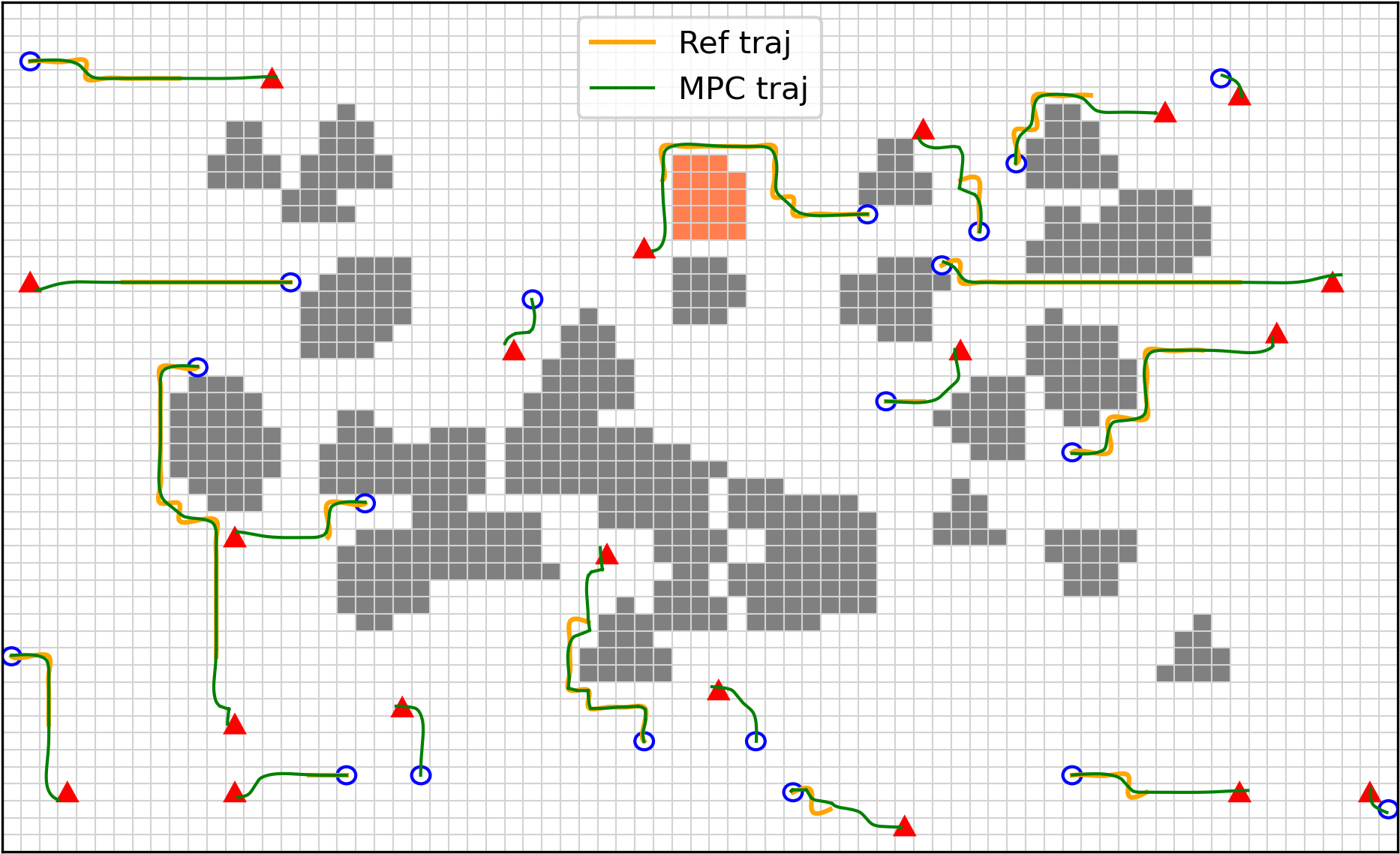}
    \caption{OT replans capturing a moved obstacle (shown in orange): (Left) before the obstacle is moved. (Right) after relocation and trajectory replan.}
    \label{fig:replan}
\end{figure}

\vspace{-0.2cm}
\section{Conclusions}\label{sec_conc}
In this paper, we address the problem of multi-robot target assignment and path planning using optimal transport (OT). Rather than directly matching individual robots to targets, we formulate the task as a mass reconfiguration problem, transforming the entire space~$\Omega$ from an initial to a final distribution. We develop conditions under which cost-optimal, non-overlapping paths exist, and introduce modifications to handle cases where such paths cannot be realized. We incorporate temporal considerations through periodic replanning and receding horizon MPC controllers. The proposed methodologies offer a novel framework for multi-agent path finding (MAPF) combined with optimal transport, laying the groundwork for incorporating control-based constraints, modulating the spatial structure and distribution of paths, and leveraging modern advances in optimal transport.

\bibliography{Refs}

\end{document}